\pdfoutput=1
\documentclass[letterpaper]{article}
\usepackage{aaai}
\usepackage{times}
\usepackage{helvet}
\usepackage{courier}
\usepackage{algorithm}
\usepackage[noend]{algpseudocode}
\usepackage{bbm}
\usepackage{url}
\usepackage{wrapfig}
\usepackage{bm}
\usepackage{color}
\usepackage{graphicx}
\usepackage{amsthm}
\usepackage{amsmath}
\usepackage{amssymb}
\usepackage{caption}
\usepackage{subcaption}
\usepackage{adjustbox}
\usepackage{enumitem}
\usepackage{lipsum}
\graphicspath{ {./figs/} }

\newtheorem{proposition}{Proposition}
\newtheorem{corollary}{Corollary}

\newcommand*\samethanks[1][\value{footnote}]{\footnotemark[#1]}

\begin{document}
%
\title{Conservativeness of untied auto-encoders}
\author{
    Daniel Jiwoong Im\thanks{Authors constributed equally.}\\
Montreal Institute for Learning Algorithms\\
University of Montreal\\
Montreal, QC, H3C 3J7\\
\texttt{imdaniel@iro.umontreal.ca}\\
\And
Mohamed Ishmael Diwan Belghanzi\samethanks[1] \\
HEC Montreal\\
3000 Ch de la Cte-Ste-Catherine\\
Montreal, QC, H3T 2A7\\
\texttt{mohamed.2.belghazi@hec.ca}\\
\And
Roland Memisevic \\
Montreal Institute for Learning Algorithms\\
University of Montreal\\
Montreal, QC, H3C 3J7\\
\texttt{roland.memisevic@umontreal.ca}}

\maketitle
\begin{abstract}
\begin{quote}
We discuss necessary and sufficient conditions for an auto-encoder to define a
conservative vector field, in which case it is associated with an
energy function akin to the unnormalized log-probability of the data.
We show that the conditions for conservativeness are more general than for 
encoder and decoder weights to be the same (``tied weights''), and that they also 
depend on the form of the hidden unit activation function, but 
that contractive training criteria, such as denoising, will enforce these 
conditions locally. 
Based on these observations, we show how we can use auto-encoders to extract the 
conservative component of a vector field.
\end{quote}
\end{abstract}
\section{Introduction}
An auto-encoder is a feature learning model that learns to reconstruct its inputs by going 
though one or more capacity-constrained ``bottleneck''-layers. 
Since it defines a mapping $r: \mathbb{R}^n \rightarrow \mathbb{R}^n$, an 
auto-encoder can also be viewed 
as dynamical system, that is trained to have fixed points at the data \cite{Seung1998}. 
Recent renewed interest in the dynamical systems perspective 
led to a variety of results that help clarify the role of auto-encoders and their relationship to 
probabilistic models. For example, \cite{Vincent2008,swersky2011autoencoders} showed that 
training an auto-encoder to denoise corrupted inputs is closely related to performing score 
matching \cite{Hyvarinen_scorematching} in an undirected model. 
Similarly, \cite{Guillaume2014} showed that training the model to denoise inputs, 
or to reconstruct them under a suitable choice of regularization penalty, 
lets the auto-encoder approximate the derivative of the empirical data density. 
And \cite{Kamyshanska2013} showed that, regardless of training criterion, any auto-encoder whose weights 
are tied (decoder-weights are identical to the encoder weights) can be written as the derivative of a scalar
``potential-'' or energy-function, which in turn can be viewed as unnormalized data log-probability. 
For sigmoid hidden units the potential function is exactly identical to the free energy of an RBM, 
which shows that there is tight link between these two types of model. 

The same is not true for untied auto-encoders, for which it has not been clear whether such an 
energy function exists. It has also not been clear under which conditions an energy function 
exists or does not exist, or even how to define it in the case where decoder-weights 
differ from encoder weights. 
In this paper, we describe necessary and sufficient conditions for the existence of an energy 
function and we show that suitable learning criteria will lead to an auto-encoder that satisfies 
these conditions at least locally, near the training data. We verify our results experimentally. 
We also show how we can use an auto-encoder to extract the conservative part of a vector field.

\begin{figure*}[t]
    \includegraphics[width=1.0\textwidth]{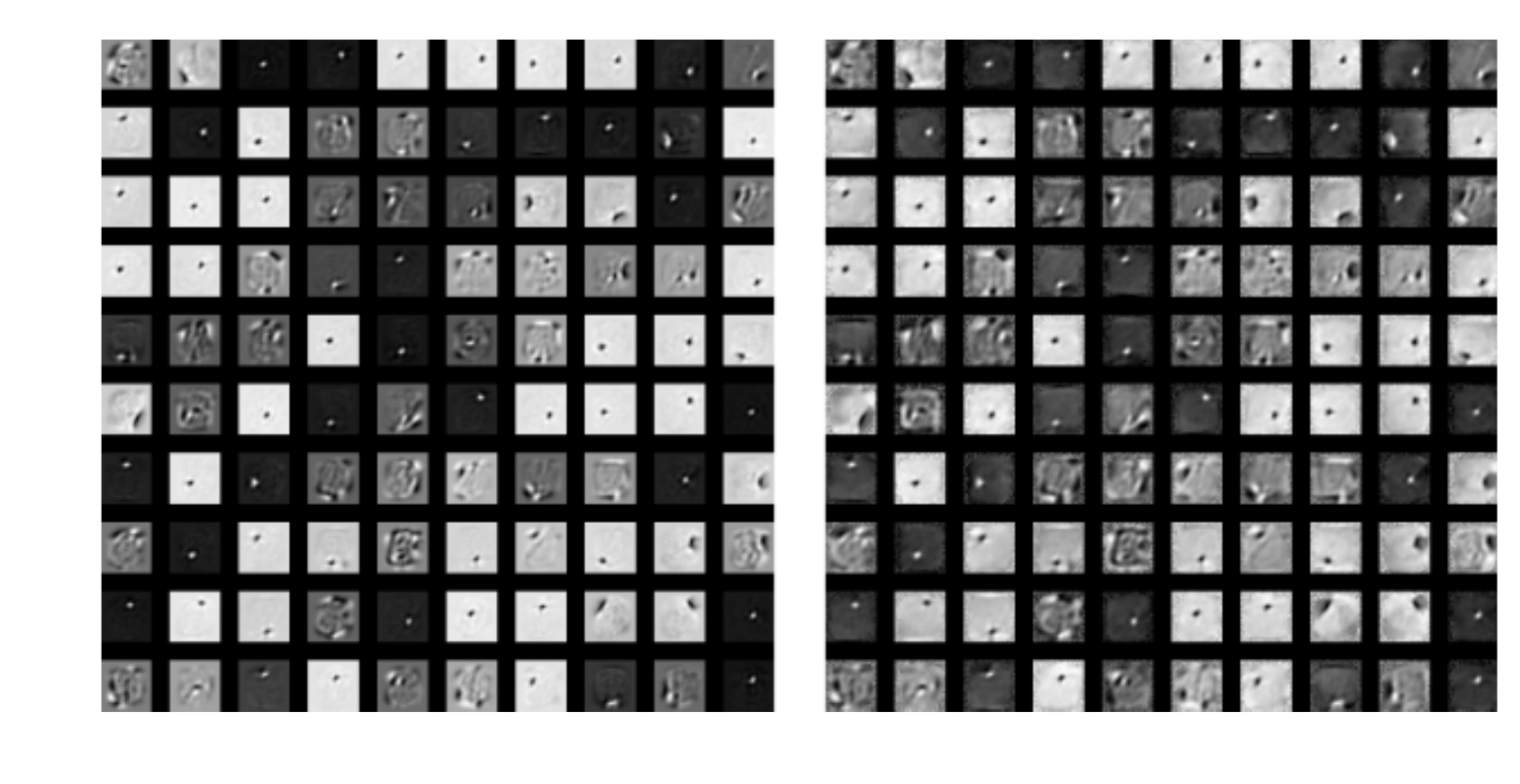}
    \vspace{-1.2cm}
    \caption{Encoder weights $W$ (left) and decoder weights $R^T$ (right).}
    \label{fig:weights}
    \vspace{-0.5cm}
\end{figure*}
\section{Background}
We will focus on auto-encoders of the form 
\begin{equation}
    r(\mathbf{x}) = Rh\big(W^\mathrm{T}\mathbf{x} + \mathbf{b} \big) + \mathbf{c}
\label{aedefinition}
\end{equation}
where $\mathbf{x}\in \mathbb{R}^n$ is an observation, $R$ and $W$ are decoder and encoder weights, 
respectively, $\mathbf{b}$ and $\mathbf{c}$ are biases, and $h(\cdot)$ is an elementwise 
hidden activation function.
An auto-encoder can be identified with its vector field, $r(\mathbf{x})-\mathbf{x}$, which is the 
set of vectors pointing from observations to their reconstructions. 
The vector field is called conservative if it can be written as the gradient of a scalar 
function $F(\mathbf{x})$, called potential or energy function:  
\begin{align}
    r(\mathbf{x}) - \mathbf{x} = \nabla F(\mathbf{x}) 
    \label{eqn:vec_field}
\end{align}
The energy function corresponds to the unnormalized probability of data.

In this case, we can integrate the vector field to find the energy function \cite{Kamyshanska2013}. 
For an auto-encoder with tied weights and real-valued observations it takes the form 
\begin{equation}
    F(\mathbf{x}) = \int h(\mathbf{u})d\mathbf{u} - \frac{1}{2}\|\mathbf{x}-\mathbf{c}\|^2_2 + \text{const}
\label{eq:energygeneral}
\end{equation}
where $\mathbf{u} = W^\mathrm{T}\mathbf{x}+\mathbf{b}$ is an auxiliary variable and $h(\cdot)$ 
can be any elementwise activation function with known anti-derivative. 
For example, the energy function of an auto-encoder with sigmoid activation function is 
identical to the (Gaussian) RBM free energy \cite{hinton2010practical}:
\begin{equation}
    F_{\text{sig}}(\mathbf{x}) = \sum_k \log \left( 1+\exp \left(W^T_{\cdot k}\mathbf{x}+b_k\right)\right) - \frac{1}{2}\|\mathbf{x}-\mathbf{c}\|^2_2 + \text{const} 
\end{equation}
A sufficient condition for the existence of an energy function is that the weights are tied 
\cite{Kamyshanska2013}, but it has not been clear if this is also necessary. 
A peculiar phenomenon in practice is that it is very common for decoder and encoder 
weights to be ``similar'' (albeit not necessarily tied) in response to training.
An example of this effect is shown in Figure~\ref{fig:weights}\footnote{We 
found these kinds of behaviours not only for unwhitened, but also for binary data.}.
This raises the question of why this happens, 
and whether the quasi-tying of weights has anything to do 
with the emergence of an energy function, and if yes, whether there is a way to compute 
the energy function despite the lack of exact symmetry. 
We shall address these questions in what follows. 
\section{Conservative auto-encoders}\label{sec:cons_ae}
One of the central objectives of this paper is understanding the conditions
for an auto-encoder to be conservative\footnote{The expressions,
``conservative vector field'' and ``conservative auto-encoders''
will be used interchangeably.} and thus to have a well-defined energy function.
In the following subsection we derive and explain said conditions.

\subsection{Conditions for conservative auto-encoders}
\begin{proposition}
\label{generalcondition}
    Consider an $m$-hidden-layer auto-encoder defined as
    \begin{multline}
        r(\mathbf{x}; \theta) = W^{(m)} h^{(m)} 
                \Big( W^{(m-1)} h^{(m-1)} \\ 
            \left(\cdots W^{(1)} h^{(1)}\left(\mathbf{x}\right)\cdots \Big) + \mathbf{c}^{(m-1)}\right) + \mathbf{c}^{(m)}  , \nonumber
    \end{multline}
    where $\theta = \cup^{m}_{k=0} \theta^{(k)}$ such that
    $\theta^{(k)}=\lbrace W^{(k)}, \mathbf{c}^{(k)} \rbrace$ are the parameters of the model,
    and $h^{(k)}(\cdot)$ is a smooth elementwise activation function at layer $k$. Then the
    auto-encoder is said to be conservative over a smooth simply connect domain $K \subseteq
    \mathbb{R}^{D}$ if and only if
    its reconstruction's Jacobian $\frac{\partial r(\mathbf{x})}{\partial
      \mathbf{x}}$ is symmetric for all $\mathbf{x} \in K$.
\end{proposition}
A formal proof is provided in the Appendix.

A region $K$ is said to be \emph{simply connected} if and only if any simple
curve in $K$ can be shrunk to a point.
It is not always the case that a region of $\mathbb{R}^{D}$ is simply
connected. For instance, a curve surrounding a punctured circle in
$\mathbb{R}^{2}$ cannot be continuously deformed to a point without crossing
the punctured region. However, as long as we make the reasonable assumption
that the activation function does not have a continuum of discontinuities, we should
not run into trouble. This makes our analysis valid for activation functions
with cusps such as ReLUs.

Throughout the paper, our focus will be on one--hidden-layer auto-encoders. Although
the necessary and sufficient conditions for their conservativeness are a special
case of the above proposition, it is worthwhile to derive them explicitly.
\begin{proposition}
\label{hatcondition}
Let $r(x)$ be a one-hidden-layer auto-encoder with $D$ dimensional inputs and $H$
hidden units,
    \begin{align*}
        r(\mathbf{x}) = R h \left(W^{T} \mathbf{x} + \mathbf{b} \right) + \mathbf{c}  , \nonumber
    \end{align*}
    where $R, W, \mathbf{b}, \mathbf{c}$ are the parameters of the
    model. Then $r(\mathbf{x})$ defines a conservative vector field over a smooth simply connect domain $K \subseteq
    \mathbb{R}^{D}$ if and only if $RD_{h'}W^T$ is symmetric for
    all $\mathbf{x} \in K$
    where $D_{h'}= \text{ diag}\left( h'(\mathbf{x}) \right)$.
\end{proposition}

\begin{proof}
Following proposition 1, an auto-encoder defines a conservative vector field if
and only if its Jacobian is symmetric for all $\mathbf{x} \in K$.
\begin{equation}
 \frac{\partial \mathbf{r}(\mathbf{x})}{\partial \mathbf{x}} = \left(\frac{\partial
    \mathbf{r}(\mathbf{x})}{\partial \mathbf{x}}\right)^T
\end{equation}
By explicitly calculating the Jacobian, this is equivalent to
\begin{equation}
 (\forall 1 \leq i < j \leq D) \, \sum_{l=0}^H \big( R_{jl}W_{li} - R_{il}W_{lj} \big)
h'_l\left(\mathbf{x}\right) = 0
\end{equation}

Defining $D_{h^{'}} = diag(h^{'}(\mathbf{x}))$, this holds if and only if
\begin{equation}
 RD_{h^{'}}W^{T} = WD_{h^{'}}R^{T} \, \label{hatequation}
\end{equation}
\end{proof}

For one-hidden-layer auto-encoders with tied weights, Equation \ref{hatequation} 
holds regardless of the choice of activation function $h$ and $\mathbf{x}$.

\begin{corollary}
\label{tiedcorollary}
An auto-encoder with tied weights always defines a conservative vector field.
\end{corollary}

Proposition 2 illustrates that the set of all one-layered tied auto-encoders is actually a subset of the set of
all conservative one-layered auto-encoders. Moreover, the inclusion is strict. That is to
say there are untied conservative auto-encoders that are not trivially equivalent
to tied ones. As example, let us compare the parametrization of tied and
conservative untied linear one-layered auto-encoders. $r_{untied}(\mathbf{x})$ in Eq.~\ref{hatequation} 
defines a conservative vector field if and only
$RW^{T}=WR^{T}$ which offers a richer parametrization than the tied linear
auto-encoder $r_{tied}(\mathbf{x}) = WW^{T}\mathbf{x}$.

In the following section we explore in more detail and generality of the
parametrization imposed by the conditions above.

\subsection{Understanding the symmetricity condition}
Note that if symmetry holds in the Jacobian of an auto-encoder's
reconstruction function, then the vector field is conservative.
A sufficient condition for symmetry of the Jacobian is that $R$ can be written
    \begin{align}
        R = CWD_{h'}E.
        \label{eqn:suff_cond3}
    \end{align}
where $C$ and $E$ are symmetric matrices, and $C$ commutes with $WD_{h'}ED_{h'}W^T$, as this will 
ensure symmetry of the partial derivatives:
\begin{align}
    \frac{\partial \mathbf{r}(\mathbf{x})}{\partial \mathbf{x}} 
    & = RD_{h'}W^T = CWD_{h'} E D_{h'}W^T \\
    & = WD_{h'}E D_{h'}W^TC 
     = WD_{h'}R^T  
     = \left(\frac{\partial \mathbf{r}(\mathbf{x})}{\partial \mathbf{x}}\right)^T\nonumber.
\end{align}

The case of tied weights ($R=W$) follows if 
$C$ and $E$ are the identity, since then
$\frac{\partial \mathbf{r}(\mathbf{x})}{\partial \mathbf{x}} = RD_{h'}W^T=WD_{h'}W^T$.

Notice that $R=CWD_{h'}$ and $R = WD_{h'}E$ are
further special cases of the condition $R=CWD_{h'}E$ when $E$ is the identity (first case)
or $C$ is the identity (second case).
Moreover, we can also find matrices $E$ and $C$ given the parameters $W$ and $R$,
which is shown in Section 1.2 of the supplementary 
material\footnote{\url{www.uoguelph.ca/~imj/files/conservative_ae_supplementary.pdf}}.
\section{Conservativeness of trained auto-encoders}\label{sec:sym}
Following \cite{Guillaume2014} we will first assume that the true data distribution is known 
and the auto-encoder is trained. We then
analyze the conservativeness of auto-encoders around fixed points of the data manifold. After that, we will proceed to 
empirically investigate and explain the tendency of trained auto-encoders to become conservative away from the data manifold. 
Finally, we will use the obtained results to explain why the product of the encoder and 
decoder weights become increasingly symmetric in response to training.  

\begin{figure*}[t]
    \begin{minipage}{0.48\textwidth}
        \includegraphics[width=1.0\textwidth]{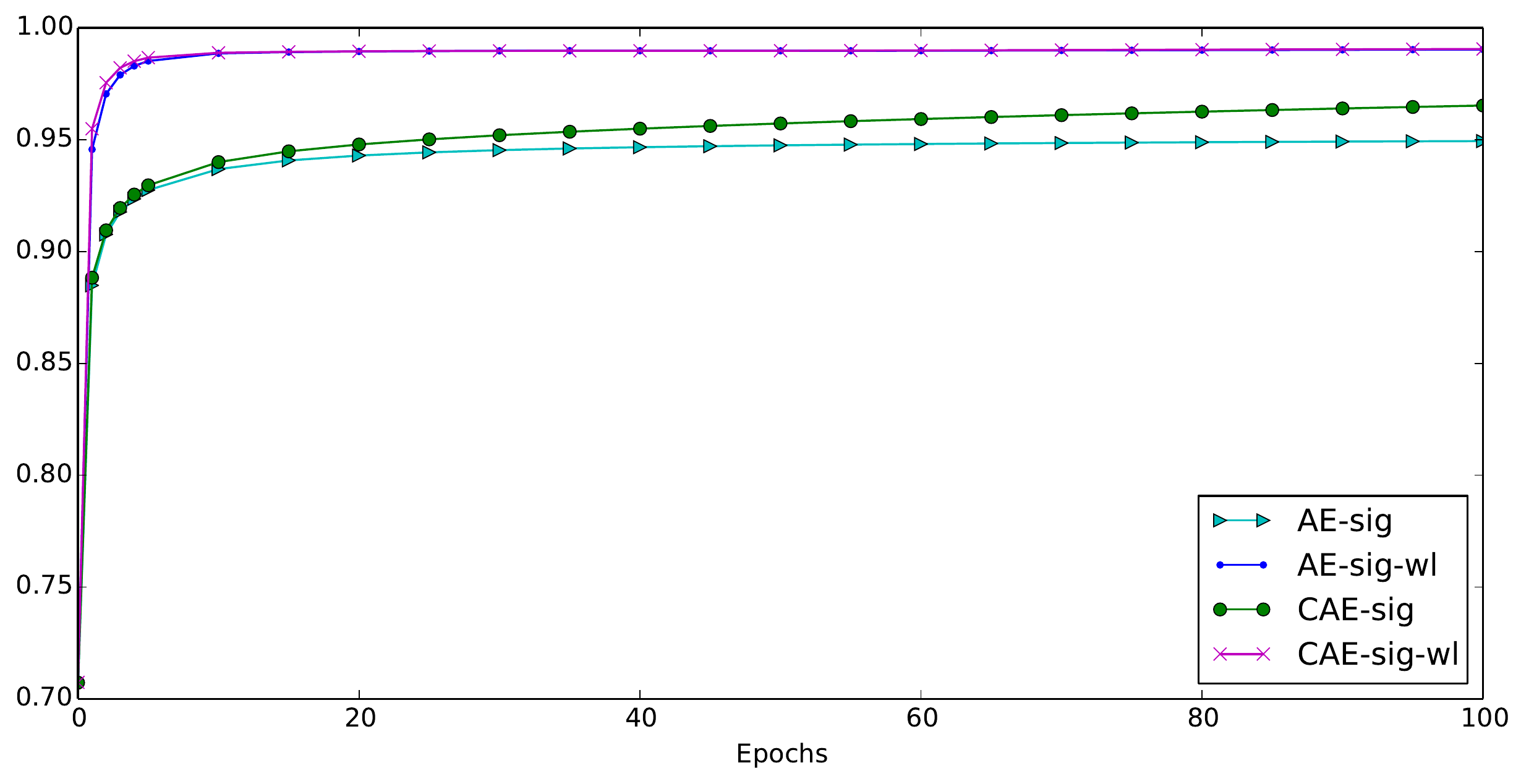}
        \vspace{-0.6cm}
        \subcaption{Symmetricity for $RD_{h'}W$ with sigmoid units}
        \label{fig:sym_sig_RDW}
    \end{minipage}
    \begin{minipage}{0.48\textwidth}
        \includegraphics[width=1.0\textwidth]{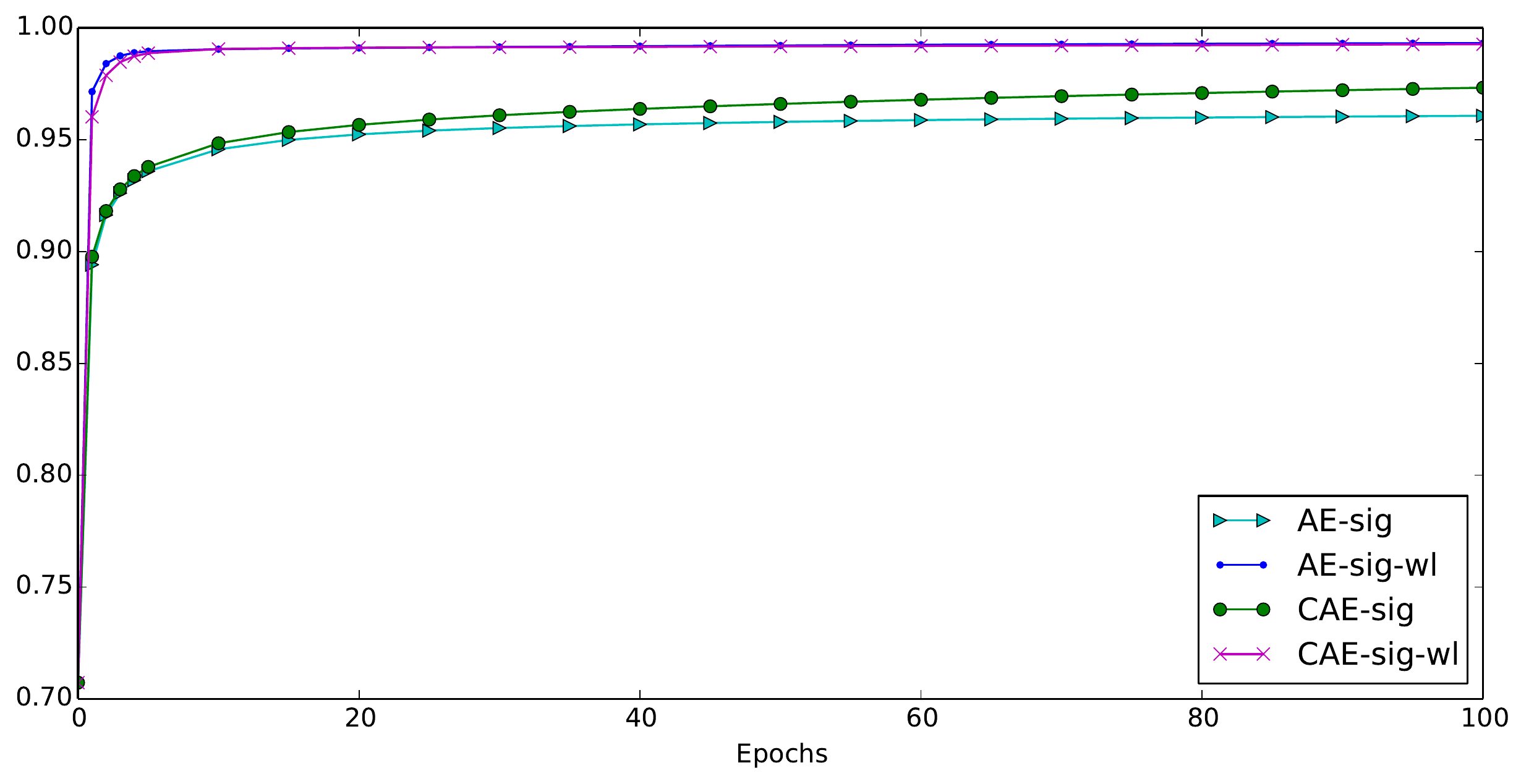}
        \vspace{-0.6cm}
        \subcaption{Symmetricity for $RW$ with sigmoid units}
        \label{fig:sym_sig_RW}
    \end{minipage}\\
    \begin{minipage}{0.48\textwidth}
        \includegraphics[width=1.0\textwidth]{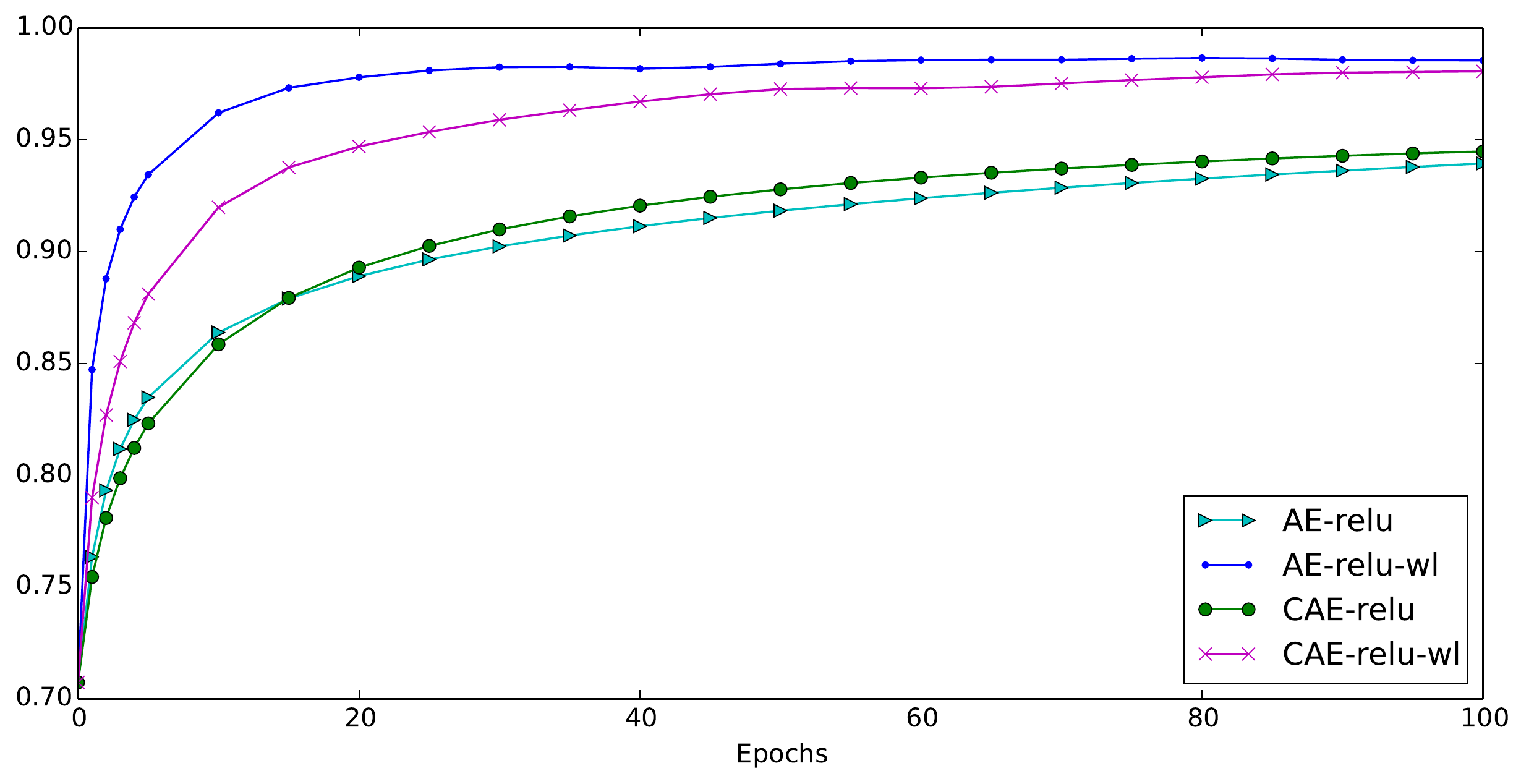}
        \vspace{-0.6cm}
        \subcaption{Symmetricity for $RD_{h'}W$ with ReLU units}
        \label{fig:sym_relu_RDW}
    \end{minipage}
    \begin{minipage}{0.48\textwidth}
        \includegraphics[width=1.0\textwidth]{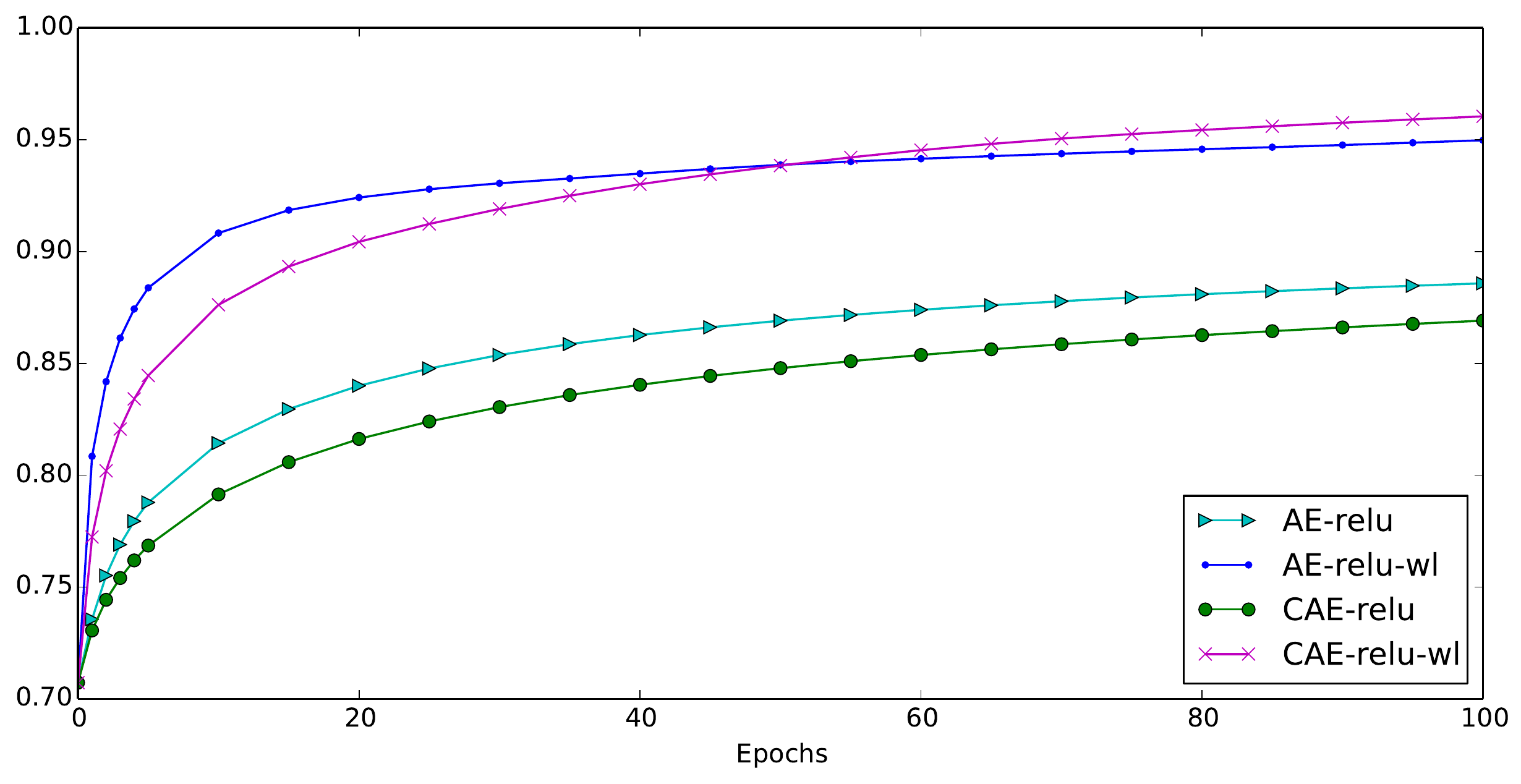}
        \vspace{-0.6cm}
        \subcaption{Symmetricity for $RW$ with ReLU units}
        \label{fig:sym_relu_RW}
    \end{minipage}
    \caption{The symmetricity distance of
        $\frac{\partial r(\mathbf{x})}{\partial \mathbf{x}}$ and
        the symmetricity distance of $RW^T$ for sigmoid activation and
        ReLU activation are illustrated over the learning time of the auto-encoder.}
    \label{fig:sym_various_activations}
    \vspace{-0.6cm}
\end{figure*}
\subsection{Local Conservativeness}
Let $r(\mathbf{x})$ be an auto-encoder that minimizes a contraction-regularized squared loss function
averaged over the true data distribution $p$,
\begin{equation}
\label{trueloss}
L_\sigma(\mathbf{x}) = \int_{\mathbb{R}^d} p(\mathbf{x})\left[ \|r(\mathbf{x})-\mathbf{x}\|^2_2
+ \bm{\epsilon} \|\frac{\partial r(\mathbf{x})}{\partial \mathbf{x}}\|^2_2\right] d\mathbf{x}
\end{equation}

A point $\mathbf{x} \, \in \mathbb{R}^{d}$ is a fixed point of the auto-encoder if and
only if $r(\mathbf{x}) = \mathbf{x}$. 

\begin{proposition}
\label{trainedconservative}
Let $r(\mathbf{x})$ be an untied one-layer auto-encoder minimizing Equation~\ref{trueloss}. Then
$r(\mathbf{x})$ is locally conservative as the contraction parameter tends to zero. 
\end{proposition}

Taking a first order Taylor expansion of
$r(\mathbf{x})$ around a fixed point $\mathbf{x}$ yields
\begin{equation}
    r(\mathbf{\mathbf{x}} + \bm{\epsilon}) = \mathbf{\mathbf{x}}
    + \frac{\partial r(\mathbf{x})}{\partial \mathbf{x}}^{T}\bm{\epsilon} +
    o(\bm{\epsilon}) \,\, \mathrm{as} \,\, \bm{\epsilon} \to 0.
\end{equation}

\cite{Guillaume2014} shows that the reconstruction $r(\mathbf{x}) - \mathbf{x}$
becomes an estimator of the score when $\|r(\mathbf{x}) - \mathbf{x} \|_{2}$ is small and the contraction parameters $\lambda \to 0$.
Hence around a fixed point we have
\begin{align}
    r(\mathbf{x} + \bm{\epsilon}) - \mathbf{x} &= \bm{\epsilon} \frac{\partial \log(p(\mathbf{x}))}{\partial \mathbf{x}}, \quad \mathrm{and}\\
    \quad \frac{\partial (r(\mathbf{x} + \bm{\epsilon}) - \mathbf{x})}{\partial \mathbf{x}}
    &= \bm{\epsilon} \frac{\partial^{2} \log(p(\mathbf{x}))}{\partial \mathbf{x}^{2}}
        \label{eqn:drdx_ddlogP}
\end{align}
where $I$ is the identity matrix.

By explicitly expressing the Jacobian of the auto-encoder's dynamics $\frac{\partial r(\mathbf{x}) - \mathbf{x}}{\partial x}$
and using the Taylor expansion of $r(\mathbf{x})$, we have 
\begin{equation}
    W^{T}D_{h'}R^{T} - I = \bm{\epsilon} \frac{\partial^{2} \log(p(\mathbf{x}))}{\partial \mathbf{x}^{2}}
    \label{eqn:HesslogP}
\end{equation}

The Hessian of $\log p(\mathbf{x})$ being symmetric, 
Equation~\ref{eqn:HesslogP} illustrates that around fixed points,
$RD_{h^{'}}W$ is symmetric.
In conjunction with {\em Proposition 2}, this shows that untied auto-encoders, when trained using 
a contractive regularizer, are locally conservative. Remark that when the auto-encoder 
is trained with patterns drawn from a continuous family, then auto-encoder forms a
continuous attractor that lies near the examples it is trained on \cite{Seung1998}.

It is worth noting that 
dynamics around fixed points can be
understood by analyzing the eigenvalues of the Jacobian.
The latter being symmetric implies that its eigenvalues cannot have complex
parts, which corresponds to the lack of oscillations one would naturally expect of a conservative
vector field. Moreover, in directions orthogonal to the fixed point, the
eigenvalues of the reconstruction will be negative. Thus the fixed point is
actually a sink.
\subsection{Empirical Conservativeness}
We now empirically analyze the conservativeness of trained untied auto-encoders.
To this end, we train an untied auto-encoder with 500 hidden units with and without
weight length constraints\footnote{Weight length constraints : 
$||\mathbf{w}_i||^2=\alpha$ for all $i=1\cdots H$ and $\alpha$ is a constant term.}
on the MNIST dataset. We measure symmetricity using 
$\text{sym}(A) = \frac{\|(A+A^T)/2\|^2}{\|A\|^2}$
which yields values between $[0,1]$ with $1$ representing complete symmetricity.
\begin{table}
    \centering
    \makeatletter\def\@captype{table}\makeatother
    \caption{Symmeticity of ADW after training AEs with 500 units on MNIST for 100 epochs. We
    denote the auto-encoders with weight length constraints as `+wl'.}
    \label{tab:sym_score}
    \begin{tabular}{|l|l|l|l|l|}
    \hline
     & ReLU & ReLU+wl & sig. & sig.+wl  \\\hline
    AE & 95.9\% & 98.7\% & 95.1\% & 99.1\%\\
    CAE & 95.2\% & 98.6\% & 97.4\% & 99.1\%\\\hline
    \end{tabular}
    \vspace{-0.4cm}
\end{table}
%
Figure~\ref{fig:sym_sig_RDW} and \ref{fig:sym_relu_RDW} shows the evolution of the symmetricity 
of $\frac{\partial r(\mathbf{x})}{\partial \mathbf{x}}=RD_{h'}W$ during training.
For untied auto-encoders, we observe that the Jacobian becomes increasingly
symmetric as training proceeds and hence, by {\em Proposition 2}, the
auto-encoder becomes increasingly conservative.

The contractive auto-encoder tends more towards symmetry than the unregularized auto-encoder.
The reach plateaus around $0.951$ and $0.974$ respectively.
It is interesting to note that auto-encoders with weight length constraints
yield sensibly higher symmetricity scores 
as shown in Table~\ref{tab:sym_score}. 
The details of the experiments and further interpretations are provided in the supplementary material.

\begin{figure*}[t]
    \begin{minipage}{0.33\textwidth}
        \includegraphics[width=1.0\textwidth]{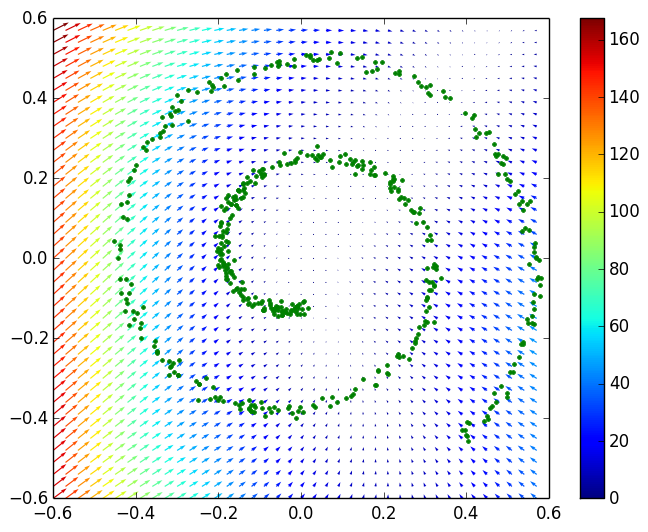}
        \vspace{-0.6cm}
        \subcaption{{\footnotesize Initial vector field}}
    \end{minipage}
    \begin{minipage}{0.33\textwidth}
        \includegraphics[width=1.0\textwidth]{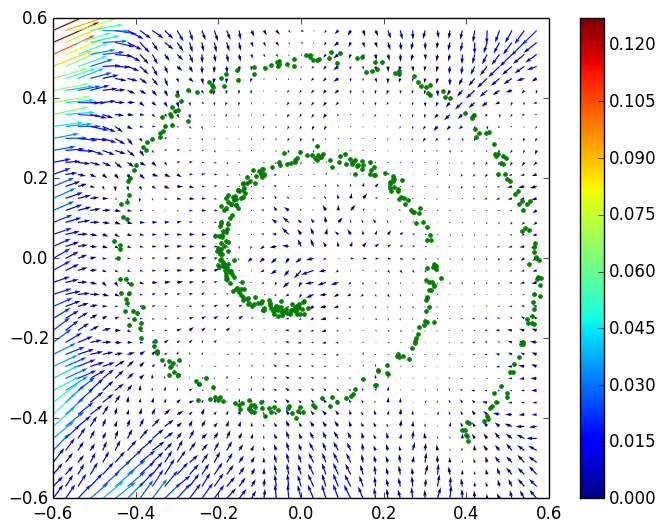}
        \vspace{-0.6cm}
        \subcaption{{\footnotesize Final vector field}}
    \end{minipage}
    \begin{minipage}{0.33\textwidth}
        \includegraphics[width=1.0\textwidth]{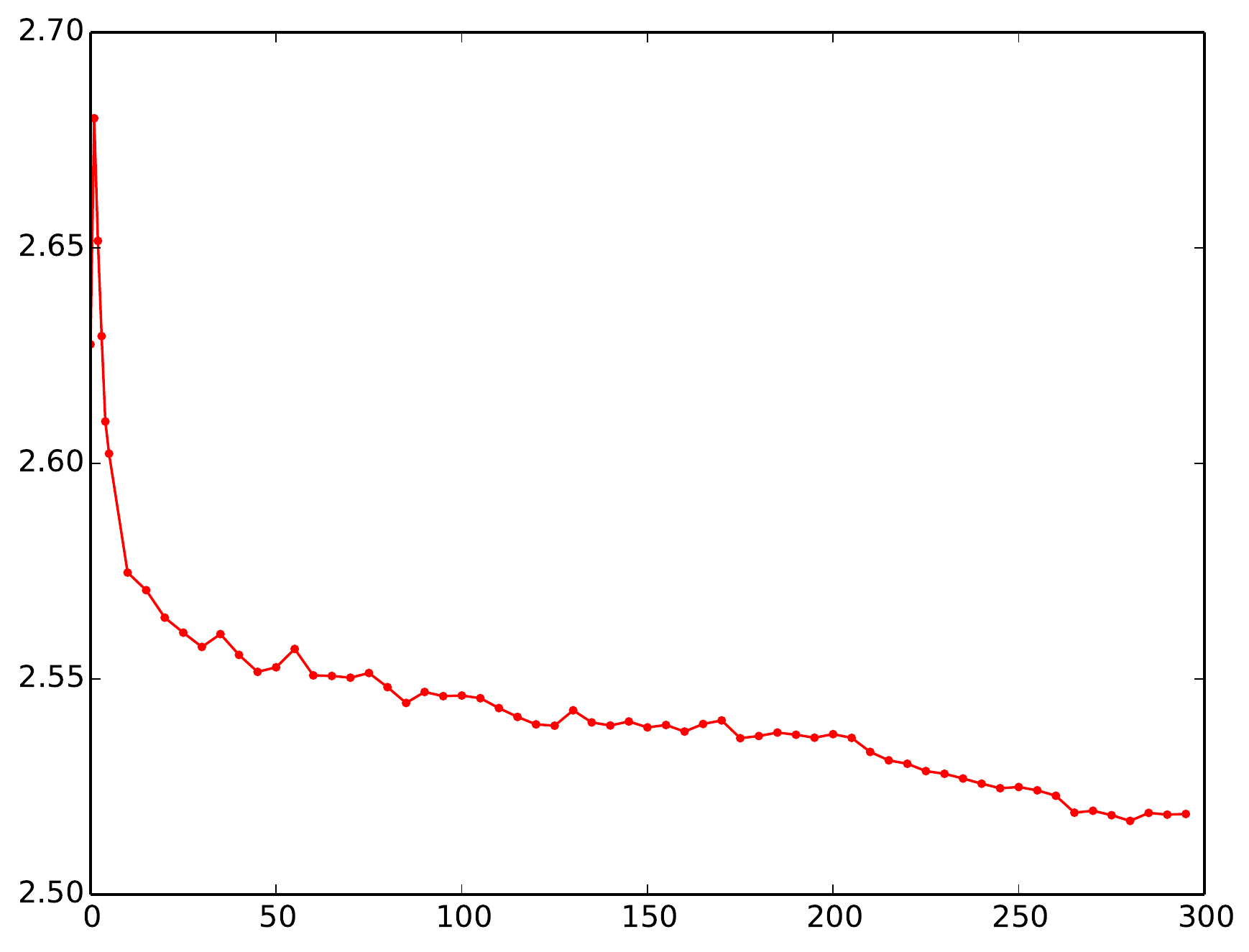}
        \vspace{-0.6cm}
        \subcaption{{\footnotesize Magnitude of curl during training}}
    \end{minipage}\\
    \caption{Initial and final vector field after training untied auto-encoder
         on spiral dataset.}
    \label{fig:spiral_vf}
    \vspace{-0.4cm}
\end{figure*}
To explicitly confirm the conservativeness of the auto-encoder in 2D, we monitor
the curl of the vector field during training.
In our experiments, we created three 2D synthetic datasets by adding gaussian
white noise to the parametrization of a line, a circle, and a spiral.
As shown in Figure~\ref{fig:spiral_vf}, we notice that the curl decrease very
sharply during training, which further demonstrates how 
untied auto-encoders become more conservative during training.
Hence, together with symmetricity measurement and decay of curliness
advocates that vector fields near the data manifold has the tendancy of becoming
conservative.
More results on line, circle, and spiral synthetic datasets can be found in the
supplementary materials.

\subsection{Symmetricity of weights product}
The product of weight $RW^T$ tends to become increasingly
symmetric during training. This behavior is more marked for sigmoid activations than for 
ReLUs as shown in Figures~\ref{fig:sym_sig_RW} and \ref{fig:sym_relu_RW}.
This can be explained by considering the Jacobian symmetricity. We approximately have
\begin{align}
    \sum^{H}_{l=1} ( R_{il}W_{lj} - R_{jl}W_{li} ) h'_l(\mathbf{x}) =0, \forall 1\le i, j \le d
    \label{eqn:nec_cond_repeat}
\end{align}

This implies that the activations of sigmoid hidden units, at least for
training data points, are independent of $h'(\mathbf{x})$ or a constant.

As shown in the supplementary material, most hidden unit activities are concentrated 
in the highest curvature region when training with 
weight length constraints. 
This forces $h_l(\mathbf{x})$ to be concentrated on high curvature regions of the
sigmoid activation. 
This may be due to either 
$h'_l(\mathbf{x})$ being nearly constant for all $l$ given $\mathbf{x}$, 
or $h'_l(\mathbf{x})$ being close to linearly independent.
In both cases, the Jacobian becomes close to the identity and hence $RW^T\approx WR^T$. 
\section{Decomposing the Vector Field}
In this section, we consider finding the closest conservative vector field, 
in a least square sense, to a non-conservative vector field.
Finding this vector field is of great practical importance in many areas of science and
engineering \cite{Bhatia2013}. Here we show that conservative auto-encoders can provide a powerful, 
deep learning based perspective onto this problem. 

The fundamental theorem of vector calculus, also known as Helmhotz
decomposition 
states that any vector field in $\mathbb{R}^3$
can be expressed as the orthogonal sum of an irrotational and a solenoidal
field. The Hodge decomposition is a generalization of this result to high
dimensional space \cite{james1966}. A complete statement of the result requires careful analysis
of boundary conditions as well as differential form formalism. But since 1-forms
correspond to vector field, and our interest lies in the latter, we abuse
notation to state the result in the special case of 1-forms as 
\begin{equation}
    \omega = d\alpha + \delta \beta + \gamma
    \label{eqn:hodge_decomposition}
\end{equation}
where $d$ is the exterior derivative, $\delta$ the co-differential, and $\Delta \gamma = 0$
\footnote{For Laplace-deRham, $\Delta=d\delta+\delta d$. Standard $\Delta$ on 1-forms
is $d\delta$.}.
This means that any 1-form (vector field) can be orthogonally decomposed into a
direct sum of a scalar,
solenoidal, and harmonic components.

This shows that it is always theoretically possible to get the closest
conservative vector field, in a least square sense, to a non-conservative one. 
When applied to auto-encoders, this guarantees the existence of a best
approximate energy function for any untied conservative auto-encoder.
For a more detailed background on the vector field decomposition we refer to the supplementary material.

\begin{figure*}[htp]
    \begin{minipage}{0.32\textwidth}
        \includegraphics[width=1.0\textwidth]{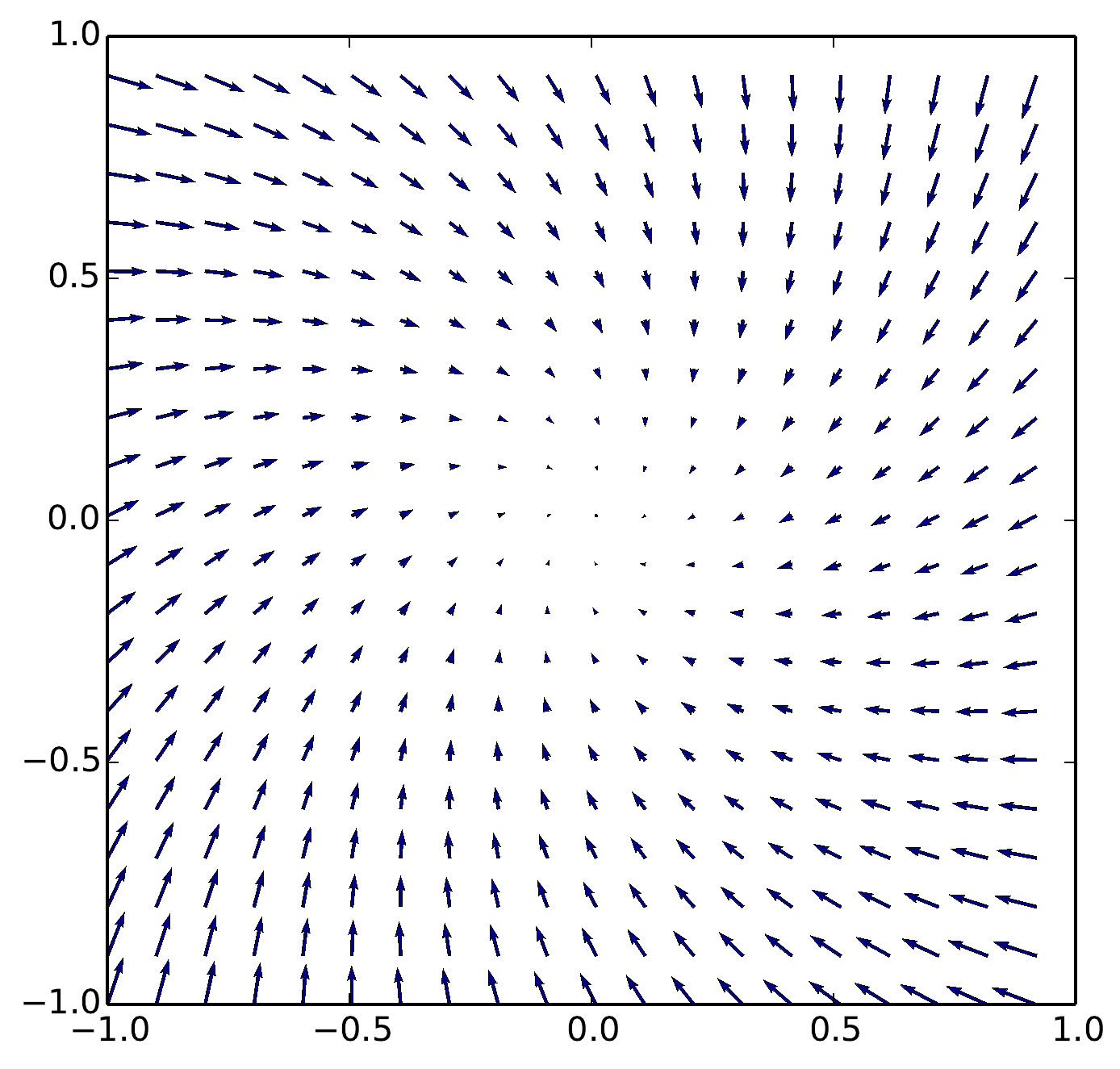}
        \vspace{-0.6cm}
    \end{minipage}
    \begin{minipage}{0.32\textwidth}
        \includegraphics[width=1.0\textwidth]{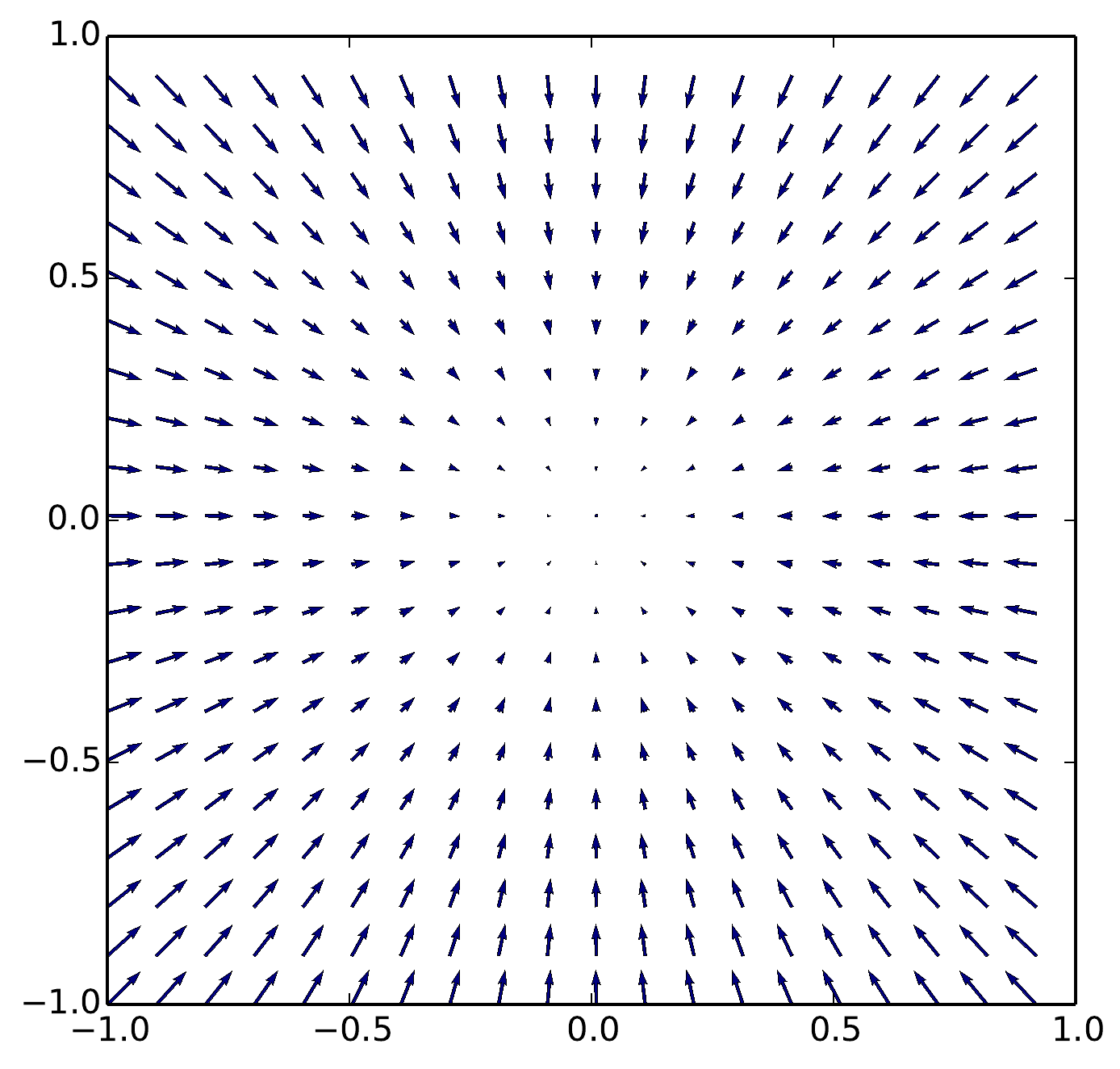}
        \vspace{-0.6cm}
    \end{minipage}
    \begin{minipage}{0.32\textwidth}
        \includegraphics[width=1.0\textwidth]{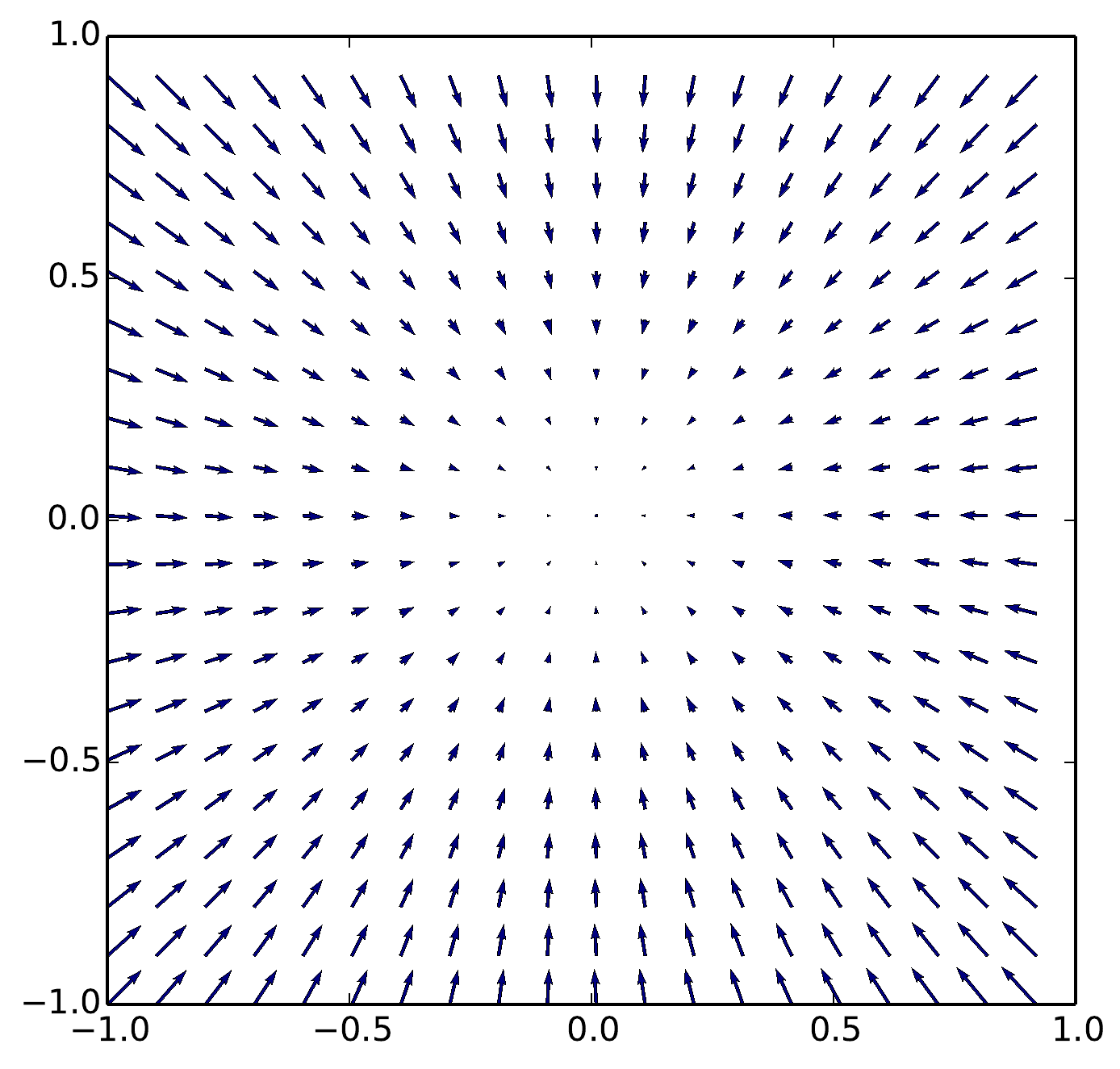}
        \vspace{-0.6cm}
    \end{minipage}\\
    \caption{Vector field learning by tied (Middle) and united (Right) auto-encoder
        on 2D unconservative vector field (Left). }
    \label{fig:hodge_vf}
    \vspace{-0.3cm}
\end{figure*}

\subsection{Extracting the Conservative Vector Field through Learning}
Although the explicit computation of the projection might be theoretically
possible in special cases, we propose to find the best approximate conservative vector 
through \emph{learning}. 
There are several advantages to learning the conservative part of a vector field:  
i) 
Learning the scalar vector field component $\alpha$ from some
vector field $\omega$ with an auto-encoder is straightforward due to the intrinsic tendency of
the trained auto-encoder to become conservative, 
ii) although there is a large body of literature to explicitly compute the projections,  
these methods are highly sensitive to boundary conditions \cite{Bhatia2013},  
while learning based methods eschew this difficulty. 

The advantage of deep learning based methods over existing approaches, such 
as matrix-valued radial basis function kernels \cite{Macedo}, is that they can be trained on 
very large amounts of data. 
To the best of our knowledge, this is the first application of neural
networks to extract the conservative part of any vector field, effectively recovering the 
scalar part of Eq.~\ref{eqn:hodge_decomposition}.
\subsubsection{Two Dimensional space}
As a proof of concept, we first extract the conservative part of a two dimensional
vector field $F(x, y) = (-x + y, -x - y)$. The field corresponds to a spiralling
sink. We train an untied auto-encoder with $1000$ ReLU 
units for $500$ epochs using BFGS over an equally spaced grid of $100$ points in each
dimension. Figure~\ref{fig:hodge_vf} clearly shows that the conservative part is
perfectly recovered.

\subsubsection{High Dimensional space}
We also conducted experiments with high dimensional vector fields. 
We created a continum of vector fields by considering convex combinations of a conservative 
and a non-conservative field.
The former is obtained by training a tied auto-encoder on MNIST and the latter by 
setting the parameters of an auto-encoder to random values. 
That is, we have 
$(W_i, R_i) = \beta (W_0, R_0) + (1-\beta) (W_K,R_K)$
where $(W_0,R_0)$ is the non-conservative auto-encoder
and $(W_K,R_K)$ is the conservative auto-encoder.
We repeatedly train a tied auto-encoder on this continuum in order to learn its conservative part.
The pseudocode for the experiment is presented in Algorithm~\ref{algo:LACV}.
\begin{algorithm}[htp]
    \vspace{-0.1cm}
    \caption{\small{Learning to approximate a conservative field with an auto-encoder}}
    \label{algo:LACV}
    {\small
    \begin{algorithmic}[1]
        \Procedure{}{$\mathcal{D}$ be a data set }
        \State{Let $(W_0, R_0)$ be a random weights for AE.}
        \State{Let $(W_K, R_K)$ be trained AE on $\mathcal{D}$.}
        \State{Generate $F_i$ $\forall i=1\cdots K$ as follows:
            \begin{itemize}
                \item $(W_i,R_i)=\beta (W_0, R_0)+(1-\beta)(W_K,R_K)$
                \item Sample $\mathbf{x}_i$ from uniform distributon in the data space.
                \item $\mathcal{F}_i = \lbrace (\mathbf{x}_i, r(\mathbf{x}_i)) \text{for} i=1\cdots N \rbrace$
            \end{itemize}}
        \For {each vector field $F_i$,}
            \State{Train a tied Auto-encoder on $F_i$}
            \State{Compute $E(\mathbf{x})$ where $ \mathbf{x} \in \mathcal{D}$}
            \State{Compute $E(\mathbf{\tilde{x}})$ where $ \mathbf{\tilde{x}}\sim\text{Binomial} $}
            \State{Count number of $E(\mathbf{x})>E(\mathbf{\tilde{x}})$.}
        \EndFor
        \EndProcedure
    \end{algorithmic}}
    \vspace{-0.1cm}
\end{algorithm}
Figure~\ref{fig:mre_hodges} shows the mean squared error as a function of training epoch 
for different values of $\beta$.
We observe that the auto-encoder's loss function decreases as $\beta$ gets closer to $1$.
This is due to auto-encoder only being able to learn the conservative component of the
vector field. 
%
We then compare the unnormalized model evidence of the auto-encoders. 
The comparison is based on computing the potential energy of auto-encoders given two points at a time.
These two points are from the MNIST and a corrupted version of the latter using salt and pepper noise. 
We validate our experiments by counting
the number of times where $E(\mathbf{x}) > E(\mathbf{x}_{\text{rand}})$. 
Given that the weights $(W_K,R_K)$ of the conservative auto-encoder are obtained by training it on MNIST, the potential energy at MNIST data points should be higher than that at the corrupted MNIST data points. 
However, this does not hold for $\beta < 1$. Even for $\beta=0.6$, we can recover the conservative component of the vector field up to $93\%$ .
Thus, we conclude that the tied auto-encoder is able to learn the conservative component of the vector field.
The procedure is detailed in Algorithm~\ref{algo:LACV}.
\begin{figure}[htp]
    \vspace{-.2cm}
    \begin{minipage}{0.44\textwidth}
    \centering
    \includegraphics[width=1.0\textwidth]{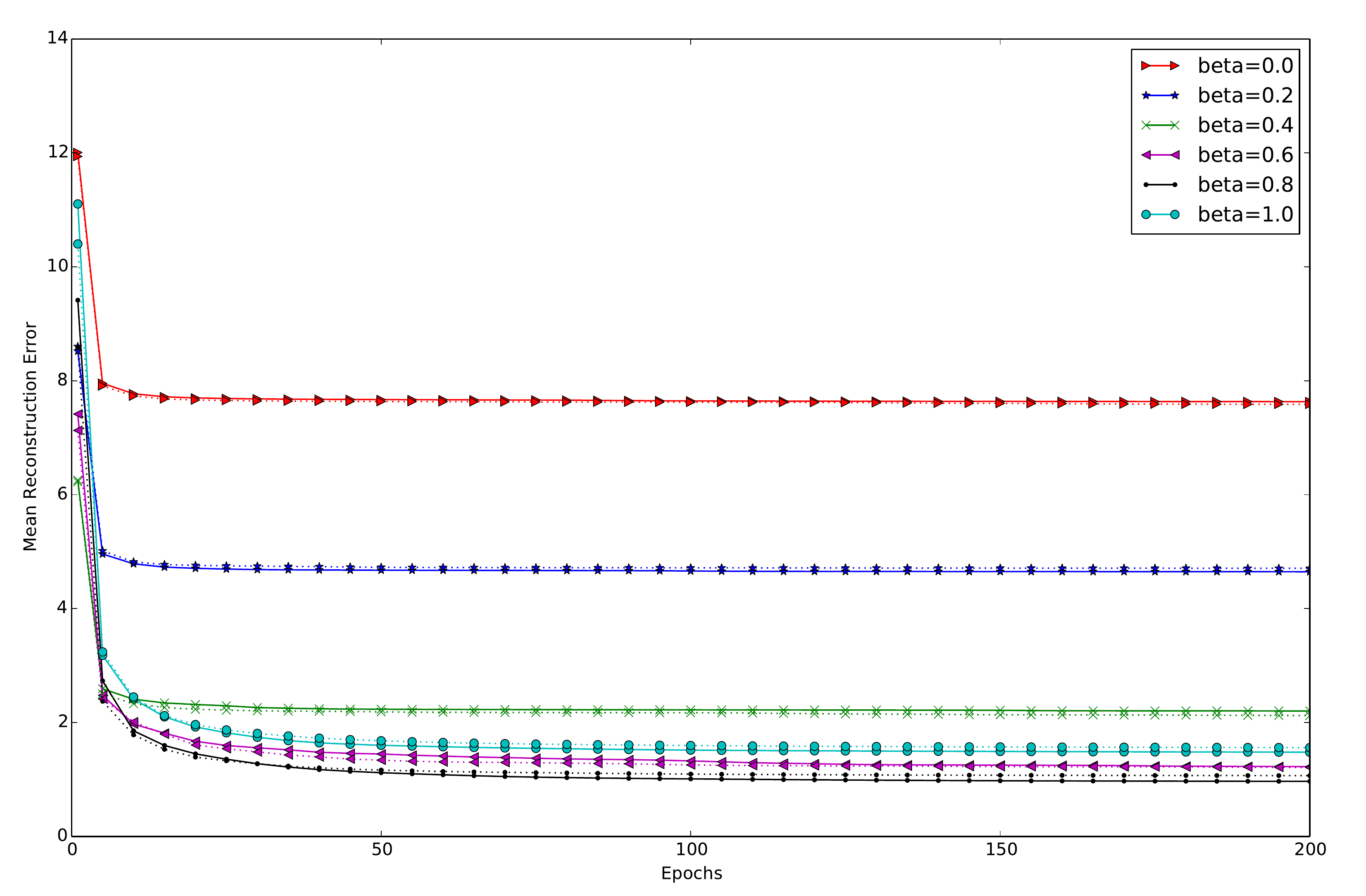}
    \caption{Learning curves for tied (dashed) and untied (solid) auto-encoders.}
    \label{fig:mre_hodges}
    \vspace{-0.5cm}
    \end{minipage}
\end{figure}
\begin{table*}[htp]
    \centering
    \caption{The fraction of observations with
        $E(\mathbf{x}) > E(\mathbf{x}_{\text{rand}})$ for different $\beta$ values.}
    \label{tab:hodge_exp}
    \vspace{-0.1cm}
    \begin{tabular}{ | l | c c c c c r | }
        \hline
        $\beta$           & 0.0 & 0.2 & 0.4 & 0.6 & 0.8 & 1.0 \\\hline
        CVF$=$Tied AE     & 0.5036 & 0.7357 & 0.9338 & 0.98838 & 0.9960 & 0.9968\\
        CVF$=$Untied AE   & 0.5072 & 0.7496 & 0.9373 & 0.98595 & 0.9958 & 0.9968\\
        \hline
    \end{tabular}
    \vspace{-0.4cm}
\end{table*}

Table~\ref{tab:hodge_exp} shows that, on average, the auto-encoders potential energy increasingly favors the original MNIST point over the corrupted ones as the vector field $F_i$ moves from $0$ to $K$.
``CVF$=$Tied AE'' refers to conservative vector field $F_K$ trained by
tied auto-encoder and ``CVF$=$Untied AE'' refers to
conservative vector field $F_K$ trained by untied auto-encoder.
\vspace{-.2cm}
\section{Discussion}
In this paper we derived necessary and sufficient conditions for autoencoders to be conservative, 
and we studied why the Jacobian of the autoencoder tends to become symmetric during training.  
Moreover, we introduced a way to extract the conservative component of a vector field
based on these properties of auto-encoders.  

An interesting direction for future research is the use of annealed importance sampling or 
similar sampling-based approaches to globally normalize the energy function values obtained 
from untied autoencoders. 
Another interesting direction is the use of parameterizations during training that 
will automatically satisfy the sufficient conditions for conservativeness but are less 
restrictive than weight tying.  
\section{Appendix}
\setcounter{proposition}{0}
\subsection{Conservative auto-encoders}
This section provides detailed derivations of {\em Proposition 1}
in Section 3.
\begin{proposition}
\label{generalcondition}
    Consider an $m$-hidden-layer auto-encoder defined as
    \begin{multline}
        r(\mathbf{x}; \theta) = W^{(m)} h^{(m)} 
                \Big( W^{(m-1)} h^{(m-1)} \\ \left(\cdots 
                    W^{(1)} h^{(1)}\left(\mathbf{x}\right)\cdots \Big) + \mathbf{c}^{(m-1)}\right) + \mathbf{c}^{(m)}  , \nonumber
    \end{multline}
    where $\theta = \cup^{m}_{k=0} \theta^{(k)}$ such that
    $\theta^{(k)}=\lbrace W^{(k)}, \mathbf{c}^{(k)} \rbrace$ are the parameters of the model,
    and $h^{(k)}(\cdot)$ is a smooth elementwise activation function at layer $k$. Then the
    auto-encoder is said to be conservative over a smooth simply connect domain $K \subseteq
    \mathbb{R}^{D}$ if and only if
    its reconstruction's Jacobian $\frac{\partial r(\mathbf{x})}{\partial
      \mathbf{x}}$ is symmetric for all $\mathbf{x} \in K$.
\end{proposition}
The high level idea is that simply finding the anti-derivative of an auto-encoder vector field
as proposed in \cite{Kamyshanska2013} does not work for untied auto-encoders. This is
due to the difference in solving first order ordinary differential equations for tied
auto-encoders and first order partial differential equations for untied
auto-encoders. Therefore, here we present a different approach that
uses differential forms to facilitate the derivation of the
existence condition of a potential energy function in the case of untied
auto-encoders.

The advantage of differential forms is that they allow us to work with a generalized, 
coordinate free system.
A differential form $\alpha$ of degree $l$ ($l$-form) on a smooth domain
$K\subseteq \mathbb{R}^d$ is an expression:
\begin{equation}
    \alpha = \sum^D_{i=1} f_idx_i.
\end{equation}
Using differential form algebra and exterior derivatives, we can show that the 
1-form implied by an untied auto-encoder is exact, which means that $\alpha$
can be expressed as
$\alpha=d\beta$ for some $\beta \in \Lambda^{l-1}(K)$.
Let $\alpha$ be the 1-form implied by the vector field of an untied auto-encoder.
Then, we have
\begin{align}
    \alpha  = \sum^D_{i=1}r_i dx_i, \text{ and }
    d\alpha = \sum^D_{i=1} d(r_i \wedge dx_i)
\end{align}
where $\wedge$ is the exterior multiplication, $d$ is the differential operatior
on differential forms, and $r(\cdot)$ is the reconstruction function of the auto-encoder.
Based on the exterior derivative properties, i) if $f\in\Lambda^0(K)$ then
$df = \sum^D_{i=1} \frac{\partial f}{\partial x_i} dx_i$ and
ii) if $\alpha\in\Lambda^l(K)$ and $\beta\in\Lambda^m(K)$ \cite{Edelen2011}
then $\alpha \beta = (-1)^{lm}\beta\alpha$,
\begin{align}
    d\alpha &= \sum^D_{i=1} d(r_i \wedge dx_i)\\
    &= \sum^D_{i,j=1} \frac{\partial r_i}{\partial x_j} ( dx_j \wedge dx_i)\\
    &= -\sum_{1\leq i< j < D} \frac{\partial r_i}{\partial x_j}dx_i \wedge dx_j +
        \sum_{1\leq i< j < D} \frac{\partial r_j}{\partial x_i}dx_i \wedge dx_j\\
    &= \sum_{1\leq i< j < D} \left(\frac{\partial r_i}{\partial x_j}
    - \frac{\partial r_j}{\partial x_i}\right) dx_i \wedge dx_j
\end{align}
According to the Poincare's theorem, which states that every exact form is closed
and convsersely, if $\alpha$ is closed then it is exact
in a simply connected region and $\alpha \in \Lambda^l(K)$,
where $\alpha$ is closed if $d\alpha=0$. Then, by Poincare's theorem,
we see that
\begin{align}
    d\alpha = \sum_{1\leq i< j < D} \left(\frac{\partial r_i}{\partial x_j}
    - \frac{\partial r_j}{\partial x_i}\right) dx_i \wedge dx_j = 0
\end{align}
This is equivalent to requiring the Jacobian to be symmetric for all
$\mathbf{x} \in K$.
{\small
\bibliography{conservativeness_ae}
\bibliographystyle{aaai}
}

\end{document}